\documentclass{article}
\usepackage{ijcai22}
\usepackage{times}
\usepackage{soul}
\usepackage{url}
\usepackage[hidelinks]{hyperref}
\usepackage[utf8]{inputenc}
\usepackage[small]{caption}
\usepackage{graphicx}
\usepackage{epstopdf}
\usepackage{amsmath}
\usepackage{amsthm}
\usepackage{booktabs}
\usepackage{algorithm}
\usepackage{algorithmic}
\usepackage{mathrsfs}
\usepackage{amsfonts}
\usepackage{multirow, bm}
\usepackage[american]{babel}
\usepackage{array}
\newtheorem{theorem}{Theorem}

\def \x {\bm{x}}

\def \w {\bm{w}}

\def \D {\mathcal{D}}

\def \Y  {\mathcal{Y}}

\def \F  {\mathcal{F}}
\def \H  {\mathcal{H}}
\def \I  {\mathbb{I}}

\def \X {\mathcal{X}}
\def \Y {\mathcal{Y}}

\def \D {\mathcal{D}}

\def \MSVMAV {{M\scriptsize SV\normalsize M\scriptsize Av} }

\newtheorem{lemma}{Lemma}
\newtheorem{definition}{Definition}

\title{On the Optimization of Margin Distribution}

\author{
Meng-Zhang Qian$^1$\and
Zheng Ai$^1$\and
Teng Zhang$^2$\And
Wei Gao$^1$\\
\affiliations
$^1$National Key Laboratory for Novel Software Technology, Nanjing University, China\\
$^2$School of Computer Science and Technology, Huazhong University of Science and Technology, China\\
\emails
\{qianmz, aiz, gaow\}@lamda.nju.edu.cn,
tengzhang@hust.edu.cn}

\begin{document}

\maketitle

\begin{abstract}
Margin has played an important role on the design and analysis of learning algorithms during the past years, mostly working with the maximization of the minimum margin. Recent years have witnessed the increasing empirical studies on the optimization of margin distribution according to different statistics such as medium margin, average margin, margin variance, etc., whereas there is a relative paucity of theoretical understanding.

In this work, we take one step on this direction by providing a new generalization error bound, which is heavily relevant to margin distribution by incorporating ingredients such as average margin and semi-variance, a new margin statistics for the characterization of margin distribution. Inspired by the theoretical findings, we propose the {M\scriptsize SV\normalsize M\scriptsize Av}, an efficient approach to achieve better performance by optimizing margin distribution in terms of its empirical average margin and semi-variance. We finally conduct extensive experiments to show the superiority of the proposed {M\scriptsize SV\normalsize M\scriptsize Av} approach.
\end{abstract}

\section{Introduction}
Margin has played an important role on the design of learning algorithms from the pioneer work \cite{Vapnik1982}, which proposed the famous Support Vector Machines (SVMs) by maximizing the minimum margin, i.e. the smallest distance from the instances to the classification boundary. \citeauthor{Boser:Guyon:Vapnik1992}~\shortcite{Boser:Guyon:Vapnik1992} introduced the kernel technique for SVMs to relax the linear separation. Large margin has been one of the most important principles on the design of learning algorithms in the history of machine learning \cite{Cortes:Vapnik1995,Schapire:Freund:Bartlett:Lee1998,Rosset:Zhu:Hastie2003,Shivaswamy:Jebara:2010,Ji:Srebro:Telgarsky2021}, even for recent deep learning \cite{Sokolic:Giryes:Sapiro:Rodrigues2017,Weinstein:Fine:HelOr2020}.

Various margin-based bounds have been presented to study the generalization performance of learning algorithms. \citeauthor{Bartlett:ShaweTaylor1999}~\shortcite{Bartlett:ShaweTaylor1999} possibly presented the first generalization margin bounds based on VC dimension and fat-shattering dimension. \citeauthor{Bartlett:Mendelson2002}~\shortcite{Bartlett:Mendelson2002} introduced the famous margin bounds based on Rademacher complexity, a data-dependent and finite-sample complexity measure. \citeauthor{Kaban:Durrant2020}~\shortcite{Kaban:Durrant2020} took advantage of geometric structure to provide margin bounds for compressive learning. \citeauthor{Gronlund:Kamma:Larsen2020}~\shortcite{Gronlund:Kamma:Larsen2020} presented the near-tight margin generalization bound for SVMs. Margin has also been an ingredient to analyze the generalization performance for other algorithms such as boosting \cite{Schapire:Freund:Bartlett:Lee1998,Breiman1999,Gao:Zhou2013}, and deep learning \cite{Bartlett:Foster:Telgarsky2017,Wei:Ma2020}.

Margin distribution has been considered as an important ingredient on the design and analysis of learning algorithms, and the basic idea is to optimize some margin statistics, relevant to the whole margin distribution rather than single margin. \citeauthor{Garg:Roth2003} \shortcite{Garg:Roth2003} introduced the model complexity measure to optimize margin distribution. \citeauthor{Pelckmans:Suykens:DeMoor2007}~\shortcite{Pelckmans:Suykens:DeMoor2007} optimized margin distribution via average margin, while \citeauthor{Aiolli:DaSanMartino:Sperduti2008}~\shortcite{Aiolli:DaSanMartino:Sperduti2008} tried to maximize the minimum margin and average margin. \citeauthor{Zhang:Zhou2014}~\shortcite{Zhang:Zhou2014} proposed the large margin distribution machine by considering average margin and margin variance simultaneously, which motivates the design of a series learning algorithms on the optimization of margin distribution \cite{Cheng:Zhang:Wen2016,Rastogi:Anand:Chandra2020}. For deep learning, \citeauthor{Jiang:Krishnan:Mobahi:Bengio2018}~\shortcite{Jiang:Krishnan:Mobahi:Bengio2018} introduced some margin distribution statistics, such as total variation, median quartile, etc., to analyze the generalization of neural networks. There is a relative paucity of theoretical understanding on how to correlate margin distribution with the generalization of learning algorithms.

This work tries to fill the gap between theoretical and empirical studies on the optimization of margin distribution, and the main contributions can be summarized as follows:
\begin{itemize}
\item We present a new generalization error bound, which is heavily relevant to margin distribution by incorporating factors such as average margin and semi-variance. Here, semi-variance is a new statistics, counting the average of squared distances between average margin and the instances' margin, that is smaller than average margin.

  \item Motivated from our theoretical result, we develop the {M\scriptsize SV\normalsize M\scriptsize Av} approach, which tries to achieve better generalization performance by optimizing margin distribution in terms of empirical average margin and semi-variance. We find the closed-form solution in optimization, and improve its efficiency via Sherman-Morrison formula.

  \item We conduct extensive empirical studies to validate the effectiveness of the {M\scriptsize SV\normalsize M\scriptsize Av} approach in comparisons with the state-of-the-art algorithms on large-margin or margin distribution optimization.
\end{itemize}

The rest of this paper is organized as follows. Section~\ref{sec:Preli} introduces some preliminaries. Section~\ref{sec:Bound} presents theoretical analysis. Section~\ref{sec:PDM} proposes the {M\scriptsize SV\normalsize M\scriptsize Av} approach. Section~\ref{sec:exp} conducts extensive empirical studies, and  Section~\ref{sec:con} concludes with future work.

\section{Preliminaries}\label{sec:Preli}
Let $\X\subseteq\mathbb{R}^d$ and $\Y=\{+1,-1\}$ denote the instance and label space, respectively. Suppose that \(\D\) is an underlying (unknown) distribution over the product space \(\X\times\Y\). Let
\[
S_n=\{(\x_1,y_1), (\x_2,y_2), \cdots, (\x_n,y_n)\}
\]
be a training sample with each element drawn independently and identically (i.i.d.) from distribution \(\D\). We use $\Pr_{\D}[\cdot]$ and $E_{\D}[\cdot]$ to refer to the probability and expectation according to distribution $\D$, respectively.

Let $\H=\{h\colon \X\to[-1,+1]\}$ be a function space. We define the classification error (or generalization risk) with respect to function $h\in\H$ and distribution \(\D\), as
\[
\mathcal{E}(h)=\Pr\nolimits_{\D}[\text{sgn}[h(\x)]\neq y]=E_{\D}[\I[yh(\x)\leq0]]\ ,
\]
where the sign function $\text{sgn}[\cdot]$ returns $+1$, $0$ and $-1$ if the argument is positive, zero and negative, respectively, and the indicator function \(\I[\cdot]\) returns $1$ when the argument is true, and $0$ otherwise.

Given an example $(\x,y)$, the \emph{margin} of $h\in \H$ is defined as $yh(\x)$, which can be viewed as a measure of the confidence of the classification. We further define the average margin of $h\in \H$ over distribution $\D$ as
\begin{equation}\label{eq:average_margin}
\theta_h=E_{(\x,y)\sim\D}[yh(\x)] \ .
\end{equation}

We also introduce the empirical Rademacher complexity \cite{Bartlett:Mendelson2002} to measure the complexity of function space $\H$ as follows:
\[
  \widehat{\mathfrak{R}}_{S_n}(\mathcal{H})=E_{\sigma_1,\sigma_2,\ldots,\sigma_n} \left[\sup_{h\in \mathcal{H}}\frac{1}{n}\sum_{i=1}^n \sigma_i h(\x_i)\right],
\]
where each $\sigma_i$ is a Rademacher variable with $\Pr[\sigma_i=+1]=\Pr[\sigma_i=-1]=1/2$ for $i\in[n]$.

We finally introduce some notations used in this work. Write \([d]=\{1,2,\ldots,d\}\) for integer \(d>0\), and $\langle \w, \x\rangle$ represents the inner product of $\w$ and $\x$. Let $\bm{I}_d$ be the identity matrix of size $d\times d$, and denote by $^\top$ the transpose of vectors or matrices. For positive \(f(n)\) and \(g(n)\), we write \(f(n)=O(g(n))\) if \(g(n)/f(n)\to c\) for constant \(c<+\infty\).

\section{Theoretical Analysis}\label{sec:Bound}
We begin with the squared margin loss as follows:
\begin{definition} For $\theta>0$, we define the squared margin loss $\ell_\theta$ with respect to function $h\in\H$ as
\begin{equation*}
\ell_\theta\big(h,(\bm{x}, y)\big)=\big[(1-yh(\bm{x})/\theta)_+\big]^2\ ,
\end{equation*}
where $(a)_+=\max(0,a)$.
\end{definition}
This is a simple extension from the traditional margin loss \cite{Bartlett:Mendelson2002}, while we consider the squared loss and unbounded constraint for the negative $yh(\bm{x})$. The margin parameter $\theta$ is generally irrelevant to learned function $h$ and data distribution in most previous theoretical and algorithmic studies.

In this work, we select margin parameter $\theta$ as the average margin when $\theta_h>0$, to correlate generalization performance with margin distribution, that is,
\[
\theta=\theta_h=E_{\D}[yh(\x)] \ ,
\]
which is dependent on data distribution and learned function.  Given training sample $S_n$, we try to learn a function $h$ by minimizing the squared margin loss as follows:
\begin{equation}
\min_{h\in\H\colon \theta_h >0} \left\{\frac{1}{n}\sum_{i=1}^n \left[\left(1-\frac{y_i h(\x_i)}{\theta_h}\right)_+\right]^2\right\}\ .\label{eq:goal_loss}
\end{equation}
For simplicity, we further introduce the notion of \emph{margin semi-variance} \cite{Markowitz:1952} as follows:
\begin{definition}
Given function $h\in\H$ and training sample $S_n$, we define the margin semi-variance as
\[
\text{SV}(h)=\frac{1}{n}\sum_{i=1}^n\left[\left(\theta_h - y_ih(\x_i)\right)_+\right]^2\ ,
\]
where $\theta_h$ denotes the average margin defined by Eqn.~\eqref{eq:average_margin}.
\end{definition}
The margin semi-variance essentially counts the average of squared deviation between average margin and the margins $y_ih(\x_i)$, which are smaller than average margin. This yields an equivalent expression for Eqn.~\eqref{eq:goal_loss} as follows:
\begin{equation*}
\min_{h\in\H:\theta_h\geq\nu > 0}\ \big\{\text{SV}(h)/\theta_h^2\big\}\ ,
\end{equation*}
that is, optimizing the squared margin loss with parameter $\theta=\theta_h$ is equivalent to minimizing margin semi-variance and maximizing average margin simultaneously.

For most real applications, we could learn some relatively-good functions from sufficient training data. Motivated from the notion of weak learner in boosting \cite{Freund:Schapire1996}, we formally define the \emph{set of relatively-good functions} for function space $\H$ as follows:
\[
\H_\nu=\{h\in\H,\theta_h\geq \nu\} \ \text{ for some small constant }\ \nu>0\ .
\]
Essentially, a relatively-good function is similar to a weak learner, which achieves slightly better performance than the randomly-guessed classifier.

\

We now present the main theoretical result as follows:
\begin{theorem}\label{thm:Main}
For small constant $\nu\geq0$, let $\H$ be a function space with relative-good set $\H_\nu$. For any $\delta\in(0,1)$ and for every $h\in\H_\nu$, the following holds with probability at least $1-\delta$ over the training sample $S_n$
\[
\mathcal{E}(h)\leq \frac{\text{SV}(h)}{\theta_h^2}+O\left(\frac{\widehat{\mathfrak{R}}_n(\H_\nu)}{\theta_h^2}+\sqrt{\frac{1}{2n}\ln\frac{4n}{\delta}} \right)\ ,
\]
with empirical Rademacher complexity $\widehat{\mathfrak{R}}_n(\H_\nu)\leq \widehat{\mathfrak{R}}_n(\H)$.
\end{theorem}

This theorem presents a new generalization error bound, which is heavily relevant to margin distribution by incorporating factors such as average margin and semi-variance. This could shed some new insights on the design of algorithms on the optimization of margin distribution as shown in Section~\ref{sec:PDM}.

The proof follows the empirical Rademacher complexity \cite{Bartlett:Mendelson2002}, while the challenge lies in the distribution-dependent average margin $\theta_h$. We solve it by constructing a sequence of intervals for average margin $\theta_h$, and the detailed proof is presented in Appendix~\ref{app:proof}.

It remains to study the empirical Rademacher complexity in Theorem~\ref{thm:Main}, and we focus on linear and kernel functions.  For instance space $\X=\{\x\in\mathbb{R}^d\colon\|\x\|\leq r\}$ and linear function space $\H=\{h(\bm{x}) = \bm{w^\top x}\colon \|\bm{w}\|\leq \Lambda\}$, let $\H_\nu$ denote the set of relatively-good classifiers. We upper bound the empirical Rademacher complexity as
\[
\widehat{\mathfrak{R}}_{S_n}(\H_\nu)\leq  \widehat{\mathfrak{R}}_{S_n}(\H)\leq r\Lambda/\sqrt{n}
\]
from the work of \cite{ShalevSchwartz:BenDavid2014}. For kernel function $\kappa(\cdot, \cdot)$, we have the kernel function space
\[
\H=\Big\{h(\bm{x}) = \sum_{i=1}^n a_i \kappa(\bm{x}_i, \bm{x})\colon\sum_{i,j=1}^n a_i a_j \kappa(\bm{x}_i, \bm{x}_j)\leq \Lambda^2\Big\}.
\]
We could upper bound the empirical Rademacher complexity for kernel functions, from \cite{Bartlett:Mendelson2002},
\[
\widehat{\mathfrak{R}}_{S_n}(\H_\nu)\leq \widehat{\mathfrak{R}}_{S_n}(\H)\leq \frac{2\Lambda}{n}\Big(\sum_{i=1}^n \kappa(\bm{x}_i, \bm{x}_i)\Big)^{1/2}\ .
\]

It is also noteworthy that the average margin $\theta_h$ is unknown on the design of new algorithms because of the unknown distribution $\D$, and we resort to the empirical average margin from training sample $S_n$ in practice.

\section{The {M\small SV\Large M\small Av} Approach}\label{sec:PDM}
Motivated from Theorem~\ref{thm:Main}, this section develops the {M\scriptsize SV\normalsize M\scriptsize Av} approach on the optimization of margin distribution, and we focus on linear and kernel functions.

\subsection{Linear Functions}\label{subsec:PDM_Linear}
For linear space $\H=\{h_{\w}(\bm{x}) = \langle \w, \x\rangle\colon \|\bm{w}\|=1\}$ and training sample $S_n$, we have the empirical average margin
\begin{equation*}
\hat{\theta}_{\w} = \frac{1}{n}\sum_{i=1}^n y_i \langle \w, \x_i\rangle\ .
\end{equation*}
For simplicity, we omit a bias term on the design of algorithm, and we will augment $\w$ and instance $\x$ with bias term $b$ and $1$ in experiments, respectively, as shown in Section~\ref{sec:exp}. Our optimization problem can be formally written as
\begin{equation*}
\min_{ \|\bm{w}\|_2=1}\Big\{\frac{\widehat{\text{SV}}(\w)}{\hat{\theta}_{\w}} \Big\}\ ,
\end{equation*}
where empirical average margin $\hat{\theta}_{\w}>0$ and empirical margin semi-variance $\widehat{\text{SV}}(\w) = \sum_{i=1}^n[(\hat{\theta}_{\w}-y_i\langle\w, \x_i\rangle)_+]^2/n$.
Obviously, this is a non-convex optimization problem, and we would optimize the empirical margin semi-variance and average margin alternatively.

Initialize the linear function $\w_0$ by optimizing empirical average margin, that is,
\begin{equation}\label{eq:initial_primal}
\bm{w}_0=\mathop{\arg\max}_{\|\w\|^2_2=1} \sum_{i=1}^n \frac{y_i \langle\w,\x_i\rangle}{n} =\sum_{i=1}^n\frac{y_i \bm{x}_i}{\|\sum_{i=1}^n y_i \bm{x}_i\|_2}\ ,
\end{equation}
where we solve $\bm{w}_0$ from its dual problem using Lagrangian function, and the details are given in Appendix~\ref{app:formulation}.

\subsubsection*{Optimization of Empirical Margin Semi-variance}
In the $k$-th iteration $(k\geq1)$ with previous classifier $\w_{k-1}$, we first calculate the empirical average margin  $\hat{\theta}_{\w_{k-1}}$ as
\begin{equation}\label{eq:empirical_average_margin1}
\hat{\theta}_{\w_{k-1}} = \frac{1}{n}\sum_{i=1}^n y_i \langle \w_{k-1}, \x_i\rangle\ .
\end{equation}
We then introduce the minimization of empirical margin semi-variance as follows:
\[
\min_{\bm{w}}\Big\{\sum_{i=1}^n\frac{[(\hat{\theta}_{\w_{k-1}}-y_i \langle \w, \x_i\rangle)_+]^2}{n}+\beta_k\left\|\w-\w_{k-1}\right\|_2^2\Big\},
\]
where $\beta_k$ is a proximal regularization parameter. We now introduce the following index set, to present a closed-form solution for the above minimization,
\begin{equation}\label{eq:index_set}
\mathcal{A}_k=\big\{i\colon y_i\langle \w_{k-1}, \x_i\rangle<\hat{\theta}_{\w_{k-1}}\quad\text{for}\quad i\in [n]\big\}\ ,
\end{equation}
i.e., the index set of instance with margins below the empirical average margin $\hat{\theta}_{\w_{k-1}}$. We can rewrite the minimization of empirical margin semi-variance as
\[
\min_{\w}\Big\{\sum_{i\in\mathcal{A}_k}\frac{(\hat{\theta}_{\w_{k-1}}-y_i \langle \w, \x_i\rangle)^2}{n}+\beta_k\left\|\w-\w_{k-1}\right\|_2^2\Big\}\ .
\]
Denote by $\w'_k$ the minimizer of the above problem, and we obtain the closed-form solution for $\w'_k$ as follows
\begin{equation}\label{eq:inverse}
\Big(\bm{I}_d + \sum_{i\in \mathcal{A}_k} \frac{\x_i \x_i^\top}{n\beta_k}\Big)^{-1}\Big(\frac{\hat{\theta}_{\w_{k-1}}}{n\beta_k}\sum_{i\in \mathcal{A}_k} y_i \x_i + \w_{k-1}\Big).
\end{equation}
One problem is to calculate the inverse in Eqn.~\eqref{eq:inverse}, which takes $O(d^3)$ computational costs ($d$ is dimensionality). This remains one challenge to deal with high-dimensional tasks.

We now present an efficient method to calculate of the inverse in Eqn.~\eqref{eq:inverse}. For simplicity, we denote by
\[
\bm{M}_k= \Big(\bm{I}_d + \sum\nolimits_{i\in \mathcal{A}_k} {\bm{x}_i \bm{x}_i^\top}\big/{n\beta_k}\Big)^{-1}\ \text{ for }\ k=1,2,\cdots,
\]
and it is easy to derive the following recursive relation:
\[
\bm{M}^{-1}_{k}=\bm{M}_{k-1}^{-1}-\sum_{i\in \mathcal{A}_{k-1}\setminus \mathcal{A}_{k}} \frac{\bm{x}_i \bm{x}_i^\top}{n\beta} + \sum_{i\in \mathcal{A}_{k}\setminus \mathcal{A}_{k-1}} \frac{\bm{x}_i \bm{x}_i^\top}{n\beta}\ ,
\]
with $\bm{M}_0=\bm{I}_d$.

We calculate $\bm{M}_{k}$ efficiently from $\bm{M}_{k-1}$ and Sherman-Morrison formula \cite{Sherman:Morrison1950}.  In other words, we initialize $\bm{M}'=\bm{M}_{k-1}$, and make the following updates iteratively, based on Sherman-Morrison formula,
\begin{eqnarray}
\bm{M}'=\bm{M}' - \frac{\bm{M}'\x_i\x_i^\top \bm{M}'}{\x_i^\top\bm{M}'\x_i - n\beta_k}\ \ \text{ for }\ \ i\in\mathcal{A}_{k-1}\setminus\mathcal{A}_k\ , \label{eq:calculate_M1}\\
\bm{M}'=\bm{M}' - \frac{\bm{M}'\x_i\x_i^\top \bm{M}'}{\x_i^\top\bm{M}'\x_i + n\beta_k}\ \ \text{ for }\ \ i\in\mathcal{A}_{k}\setminus\mathcal{A}_{k-1}\ . \label{eq:calculate_M2}
\end{eqnarray}
We then obtain $\bm{M}_k=\bm{M}'$, and the minimizer of empirical margin semi-variance is given by
\begin{equation}\label{eq:update1}
\w'_k=\bm{M}' \Big(\frac{\hat{\theta}_{\w_{k-1}}}{n\beta_k}\sum_{i\in \mathcal{A}_k} y_i \x_i + \w_{k-1}\Big).
\end{equation}

\subsubsection*{Optimization of Empirical Average Margin}
We now study the maximization of empirical average margin, which can be formalized as:
\[
\w_k=\mathop{\arg\min}_{\w}\Big\{ -\frac{1}{n}\sum_{i=1}^n y_i \langle\w, \x_i\rangle + \alpha_k\left\|\w - \w'_k\right\|_2^2\Big\},
\]
where $\alpha_k$ is a proximal regularization parameter. It is easy to obtain the closed-form solution as follows:
\begin{equation}\label{eq:update2}
\w_k=\w'_k+\frac{1}{2\alpha_k n}\sum_{i=1}^n y_i \bm{x}_i\ .
\end{equation}
We obtain $\w_k=\w_k/\|\w_k\|$ in the $k$-th iteration.  Algorithm~\ref{alg:PDM} presents a detailed description of our {M\scriptsize SV\normalsize M\scriptsize Av} approach.

\begin{algorithm}[tb]
  \caption{The {M\scriptsize SV\normalsize M\scriptsize Av} Approach}
  \label{alg:PDM}
  \textbf{Input}: Training sample $S_n$, iteration number $T$, and proximal parameters $\alpha_k$ and $\beta_k$\\
  \textbf{Output}: $\bm{w}$
  \begin{algorithmic}[1] 
  \STATE Initialize $\bm{M}_0 = \bm{I}_d$, $A_0 = \emptyset$ and $\w_0$ by Eqn.~\eqref{eq:initial_primal}
  \FOR {$k=1,2,\cdots,T$ }
  \STATE Compute empirical average margin $\hat{\theta}_{\w_{k-1}}$ by Eqn.~\eqref{eq:empirical_average_margin1}
  \STATE Solve the index set $\mathcal{A}_k$ by Eqn.~\eqref{eq:index_set}
  \STATE Compute $\bm{M}_k= \big(\bm{I}_d + \sum_{i\in \mathcal{A}_k} {\bm{x}_i \bm{x}_i^\top}\big/{n\beta_k}\big)^{-1}$ by Eqns.~\eqref{eq:calculate_M1} and \eqref{eq:calculate_M2}
  \STATE Compute the minimizer $\w_k'$ for empirical margin semi-variance by Eqn.~\eqref{eq:update1}
  \STATE Solve the empirical average margin maximizer $\w_k$ by Eqn.~\eqref{eq:update2}, and normalize $\w_k = \w_k / \left\|\w_k\right\|_2$
  \ENDFOR
  \STATE \textbf{return} $\w=\w_T$
  \end{algorithmic}
\end{algorithm}

\subsection{Kernelization}\label{subsec:PDM_Kernel}
This section focuses on kernel mapping $\bm{\phi}:\mathcal{X}\to\mathbb{H}$ for Hilbert space $\mathbb{H}$, we consider $h(\x)=\langle \w, \bm{\phi}(\x)\rangle$ with $\w\in \mathbb{H}$ and $\bm{\phi}(\x)\in\mathbb{H}$. The optimization problem is given by
\begin{equation*}
\min_{\bm{w}}\Big\{\frac{\widehat{\text{SV}}(\w)}{\hat{\theta}_{\w}} \Big\}\ ,
\end{equation*}
where average margin $\hat{\theta}_{\w}=\sum_{i=1}^n y_i\langle \w, \bm{\phi}(\x_i)\rangle/n>0$, and margin semi-variance
\[
\widehat{\text{SV}}(\w)=\frac{1}{n}\sum_{i=1}^n\left[\left(\hat{\theta}_{\w}-y_i\langle\w, \bm{\phi}(\x_i)\rangle\right)_+\right]^2\ .
\]

\begin{figure}[!t]
\centering
\includegraphics[width=0.23\textwidth]{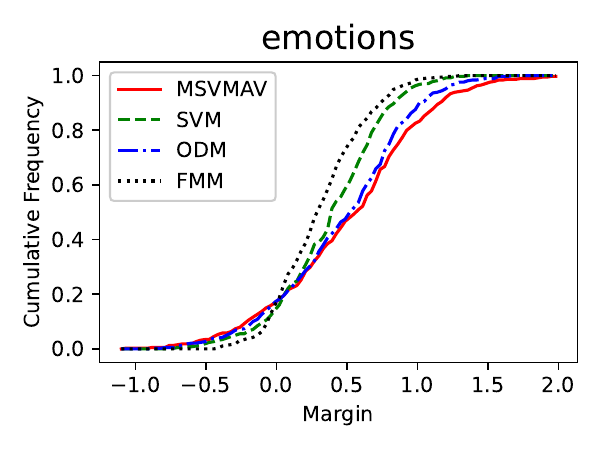}
\includegraphics[width=0.23\textwidth]{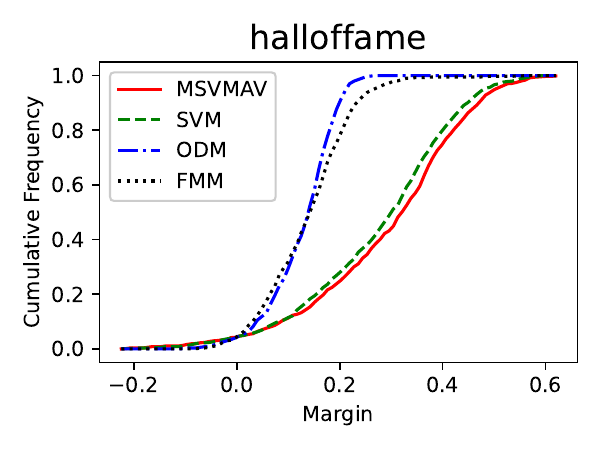}

\includegraphics[width=0.23\textwidth]{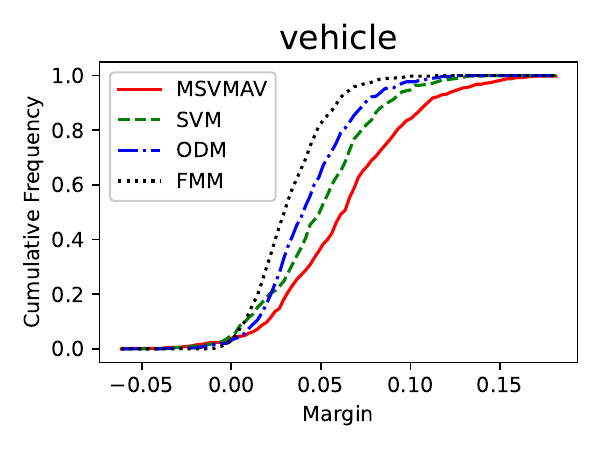}
\includegraphics[width=0.23\textwidth]{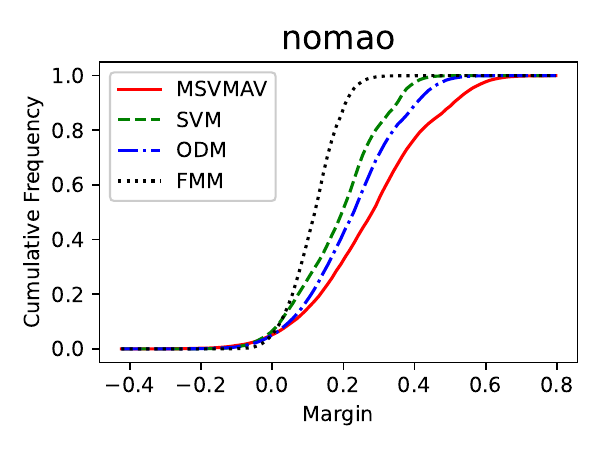}
\centering
\caption{Cumulative frequency versus margin of our {M\scriptsize SV\normalsize M\scriptsize Av} and other algorithms such as SVM, ODM and FMM. The more right the curve, the better the margin distribution.}\label{fig:margin}
\end{figure}

\begin{table*}[ht]
  \caption{Comparisons of the test accuracies (mean$\pm$std.) on 30 datasets. $\bullet/\circ$ indicates that our {M\scriptsize SV\small M\scriptsize Av} approach is significantly better/worse than the corresponding algorithms (pairwise t-tests at 95\% significance level). `N/A' indicates that LSSVM does not return results on the data set within 12 hours.
  }\label{table:PDM}
\begin{center}
\begin{small}
\setlength{\tabcolsep}{1.6mm}{
\begin{tabular}{|cccccccc|}
  \hline
  Dataset&{M\scriptsize SV\small M\scriptsize Av}&SVM&SVR&LSSVM&ODM&MAMC&FMM\\ \hline
advertise &.9837$\pm$.0015&.9823$\pm$.0021$\bullet$&.9825$\pm$.0024$\bullet$&.9819$\pm$.0019$\bullet$&.9701$\pm$.0021$\bullet$&.9132$\pm$.0007$\bullet$&.9799$\pm$.0025$\bullet$\\
australian     &.8314$\pm$.0079&.8167$\pm$.0077$\bullet$&.8171$\pm$.0048$\bullet$&.8220$\pm$.0036$\bullet$&.8104$\pm$.0065$\bullet$&.7700$\pm$.0201$\bullet$&.8295$\pm$.0100\\
bibtex         &.7417$\pm$.0040&.7378$\pm$.0069$\bullet$&.7469$\pm$.0060$\circ$&.7478$\pm$.0046$\circ$&.7469$\pm$.0053$\circ$&.6689$\pm$.0123$\bullet$&.7371$\pm$.0062$\bullet$\\
biodeg    &.8712$\pm$.0087&.8741$\pm$.0068&.8490$\pm$.0107$\bullet$&.8741$\pm$.0068&.8681$\pm$.0075&.6872$\pm$.0000$\bullet$&.8613$\pm$.0089$\bullet$\\
breastw        &.9730$\pm$.0038&.9725$\pm$.0041&.9701$\pm$.0051$\bullet$&.9684$\pm$.0051$\bullet$&.8428$\pm$.0099$\bullet$&.4088$\pm$.0000$\bullet$&.9720$\pm$.0063\\
diabetes       &.7530$\pm$.0088&.7561$\pm$.0104&.7478$\pm$.0077$\bullet$&.7491$\pm$.0078$\bullet$&.6110$\pm$.0146$\bullet$&.5974$\pm$.0000$\bullet$&.7496$\pm$.0082\\
emotions       &.8216$\pm$.0074&.7983$\pm$.0130$\bullet$&.7675$\pm$.0181$\bullet$&.8036$\pm$.0127$\bullet$&.8084$\pm$.0111$\bullet$&.6975$\pm$.0000$\bullet$&.7697$\pm$.0218$\bullet$\\
german   &.7518$\pm$.0097&.7457$\pm$.0095$\bullet$&.7525$\pm$.0107&.7530$\pm$.0106&.7502$\pm$.0095&.6800$\pm$.0000$\bullet$&.7525$\pm$.0102\\
halloffame&.9644$\pm$.0025&.9617$\pm$.0047$\bullet$&.9593$\pm$.0041$\bullet$&.9617$\pm$.0047$\bullet$&.9620$\pm$.0028$\bullet$&.9280$\pm$.0000$\bullet$&.9598$\pm$.0036$\bullet$\\
hill-valley   &.7771$\pm$.0631&.5972$\pm$.0195$\bullet$&.6999$\pm$.0572$\bullet$&.5972$\pm$.0195$\bullet$&.8898$\pm$.0079$\circ$&.5000$\pm$.0000$\bullet$&.8526$\pm$.0296$\circ$\\
kc1            &.8713$\pm$.0025&.8711$\pm$.0039&.8642$\pm$.0026$\bullet$&.8711$\pm$.0039&.8690$\pm$.0035$\bullet$&.8649$\pm$.0000$\bullet$&.8701$\pm$.0028$\bullet$\\
parkinsons     &.9222$\pm$.0277&.8923$\pm$.0388$\bullet$&.9342$\pm$.0315&.9444$\pm$.0338$\circ$&.8863$\pm$.0342$\bullet$&.7949$\pm$.0000$\bullet$&.8932$\pm$.0266$\bullet$\\
pbcseq         &.6626$\pm$.0074&.6439$\pm$.0147$\bullet$&.6595$\pm$.0104&.6555$\pm$.0104$\bullet$&.6562$\pm$.0125$\bullet$&.6700$\pm$.0187$\circ$&.6549$\pm$.0117$\bullet$\\
sleepdata      &.6925$\pm$.0056&.6743$\pm$.0051$\bullet$&.6833$\pm$.0064$\bullet$&.6712$\pm$.0124$\bullet$&.5691$\pm$.0156$\bullet$&.5659$\pm$.0000$\bullet$&.6833$\pm$.0047$\bullet$\\
students       &.8957$\pm$.0062&.8913$\pm$.0088$\bullet$&.8870$\pm$.0064$\bullet$&.8867$\pm$.0069$\bullet$&.8893$\pm$.0046$\bullet$&.5133$\pm$.0129$\bullet$&.8898$\pm$.0040$\bullet$\\
titanic        &.7658$\pm$.0038&.7636$\pm$.0000$\bullet$&.7636$\pm$.0000$\bullet$&.7636$\pm$.0000$\bullet$&.7509$\pm$.0211$\bullet$&.6386$\pm$.0000$\bullet$&.7636$\pm$.0000$\bullet$\\
tokyo1         &.9351$\pm$.0040&.9307$\pm$.0031$\bullet$&.9337$\pm$.0027$\bullet$&.9281$\pm$.0054$\bullet$&.9248$\pm$.0053$\bullet$&.7523$\pm$.0534$\bullet$&.9363$\pm$.0037\\
vehicle        &.9708$\pm$.0057&.9422$\pm$.0473$\bullet$&.9746$\pm$.0031$\circ$&.9748$\pm$.0034$\circ$&.9720$\pm$.0061&.7041$\pm$.0000$\bullet$&.9718$\pm$.0083\\
vertebra       &.8194$\pm$.0197&.7978$\pm$.0149$\bullet$&.7731$\pm$.0144$\bullet$&.7753$\pm$.0131$\bullet$&.7957$\pm$.0105$\bullet$&.7581$\pm$.0000$\bullet$&.7763$\pm$.0100$\bullet$\\
wdbc           &.9778$\pm$.0067&.9787$\pm$.0084&.9725$\pm$.0030$\bullet$&.9696$\pm$.0044$\bullet$&.9237$\pm$.0094$\bullet$&.5398$\pm$.1287$\bullet$&.9655$\pm$.0064$\bullet$\\
a9a            &.8433$\pm$.0009&.8417$\pm$.0007$\bullet$&.8403$\pm$.0007$\bullet$&.8358$\pm$.0006$\bullet$&.8430$\pm$.0012&.7577$\pm$.0000$\bullet$&.8390$\pm$.0012$\bullet$\\
acoustic&.7494$\pm$.0011&.7321$\pm$.0037$\bullet$&.7394$\pm$.0004$\bullet$&N/A$\bullet$&.7406$\pm$.0023$\bullet$&.7206$\pm$.0055$\bullet$&.7402$\pm$.0005$\bullet$\\
bank &.9057$\pm$.0004&.8960$\pm$.0008$\bullet$&.9015$\pm$.0004$\bullet$&N/A$\bullet$&.9021$\pm$.0006$\bullet$&.8854$\pm$.0000$\bullet$&.9021$\pm$.0005$\bullet$\\
eurgbp&.5332$\pm$.0027&.5042$\pm$.0061$\bullet$&.5317$\pm$.0028$\bullet$&N/A$\bullet$&.5111$\pm$.0084$\bullet$&.4985$\pm$.0000$\bullet$&.5288$\pm$.0028$\bullet$\\
jm1            &.8125$\pm$.0015&.8132$\pm$.0012&.8119$\pm$.0016&.8123$\pm$.0015&.8065$\pm$.0000$\bullet$&.8065$\pm$.0000$\bullet$&.8076$\pm$.0010$\bullet$\\
magic&.7998$\pm$.0005&.7976$\pm$.0012$\bullet$&.7928$\pm$.0008$\bullet$&.7912$\pm$.0009$\bullet$&.7943$\pm$.0007$\bullet$&.6514$\pm$.0000$\bullet$&.7987$\pm$.0014$\bullet$\\
nomao          &.9453$\pm$.0005&.9421$\pm$.0061$\bullet$&.9439$\pm$.0005$\bullet$&N/A$\bullet$&.9452$\pm$.0004&.7062$\pm$.0000$\bullet$&.9442$\pm$.0008$\bullet$\\
phishing       &.9388$\pm$.0011&.9405$\pm$.0017$\circ$&.9371$\pm$.0007$\bullet$&.9343$\pm$.0008$\bullet$&.9386$\pm$.0014&.5532$\pm$.0008$\bullet$&.9318$\pm$.0009$\bullet$\\
pol            &.9054$\pm$.0013&.8746$\pm$.0398$\bullet$&.9002$\pm$.0019$\bullet$&.9041$\pm$.0018$\bullet$&.6788$\pm$.0037$\bullet$&.6740$\pm$.0000$\bullet$&.8866$\pm$.0016$\bullet$\\
run-walk&.7260$\pm$.0007&.7169$\pm$.0000$\bullet$&.7077$\pm$.0004$\bullet$&N/A$\bullet$&.7104$\pm$.0061$\bullet$&.5431$\pm$.0757$\bullet$&.7261$\pm$.0054\\
\hline
\multicolumn{2}{|c}{Win/Tie/Loss}&23/6/1&24/4/2&23/4/3&22/6/2&29/0/1&22/7/1\\
\hline
\end{tabular}}

\end{small}
\end{center}
\end{table*}

It is intractable to solve such optimization problem directly because of high or even infinity dimensionality. According to Representer theorem \cite{Scholkopf:Smola:Bach2002}, we first have $\w^* = \sum_{i=1}^n a_i \bm{\phi}(\x_i)$, spanned by $\{\bm{\phi}(\x_i), i\in [n]\}$ with coefficients $a_1, \cdots, a_n$. This follows the prediction
\[
h(\x)=\langle \w, \bm{\phi}(\x) \rangle=\sum_{i=1}^n a_i\kappa(\x, \x_i)\ ,
\]
where $\kappa(\cdot,\cdot)$ denotes the kernel function. For simplicity, denote by $\bm{a}=(a_1, a_2, \cdots, a_n)^\top$, and write the gram matrix of instances in $S_n$ as
\[
\bm{K}=(\bm{K}_1,\bm{K}_2,\cdots,\bm{K}_n)=\big(K_{ij}\big)_{n\times n} = \big(\kappa(\bm{x}_i, \bm{x}_j)\big)_{n\times n},
\]
where $\bm{K}_i$ denotes the $k$-th column of matrix $\bm{K}$. Hence, our optimization problem can be further rewritten as
\begin{equation*}
  \min_{\bm{a}}\Big\{\frac{1}{n}\sum_{i=1}^n \left[\Big(1-\frac{y_i\langle\bm{K}_i,\bm{a} \rangle}{\hat{\theta}_{\bm{a}}}\Big)_+\right]^2\Big\}= \min_{\bm{a}}\Big\{\frac{\widehat{\text{SV}}(\bm{a})}{\hat{\theta}_{\bm{a}}} \Big\}\ ,
\end{equation*}
where the empirical average margin $\hat{\theta}_{\bm{a}}=\sum_{i=1}^n y_i \langle \bm{K}_i, \bm{a}\rangle/n$ and semi-variance $\widehat{\text{SV}}(\bm{a})=\sum_{i=1}^n[(\hat{\theta}_{\bm{a}}-y_i\langle\bm{K}_i, \bm{a}\rangle)_+]^2/n$.

We first initialize the classifier $\bm{a}_0$ by maximizing the empirical average margin as follows:
\[
  \bm{a}_0 = \mathop{\arg\max}_{\bm{a}^\top \bm{K}\bm{a}=1} \Big\{\frac{1}{n}\sum_{i=1}^n y_i \langle \bm{K}_i, \bm{a}\rangle\Big\}.
\]

In the $k$-th iteration with previous classifier $\bm{a}_{k-1}$, we minimize the margin semi-variance based on previous average margin $\hat{\theta}_{\bm{a}_{k-1}}$. We write
\[
\big\|\bm{a} - \bm{a}_{k-1}\big\|_{\bm{I}_n+\bm{K}}=\big((\bm{a} - \bm{a}_{k-1})^\top (\bm{I}_n+\bm{K}\big)\big(\bm{a} - \bm{a}_{k-1})\big)^{1/2} \ ,
\]
and the optimization problem can be given by
\[
\min_{\bm{a}}\Big\{\sum_{i=1}^n\frac{[(\hat{\theta}_{\bm{a}_{k-1}}-y_i \langle\bm{K}_i,\bm{a}\rangle)_+]^2}{n}+\beta_k \|\bm{a} - \bm{a}_{k-1}\|^2_{\bm{I}_n+\bm{K}}\Big\},
\]
where $\beta_k$ is a proximal regularization parameter. We introduce the index set $\mathcal{A}_k=\{i\colon y_i\langle\bm{K}_i,\bm{a}\rangle<\hat{\theta}_{\bm{a}_{k-1}}\text{ for } i\in [n]\}$, and obtain the empirical margin semi-variance minimizer
\begin{equation*}
  \bm{a}'_k=\bm{M}_k\Big((\bm{K} + \bm{I}_n)\bm{a}_{k-1}+\hat{\theta}_{\bm{a}'_k}\sum_{i\in \mathcal{A}_k}\frac{y_i\bm{K}_i^\top}{n\beta}\Big)
\end{equation*}
where we use the Sherman-Morrison formula to calculate 
\[
\bm{M}_k=\left(\sum_{i\in \mathcal{A}_k}\frac{\bm{K}_i^\top \bm{K}_i}{n\beta} + \bm{K} + \bm{I}_n\right)^{-1}\ .
\] 

We finally maximize the empirical average margin based on the following optimization problem:
\[
\min_{\bm{a_k}}\Big\{ -\frac{1}{n}\sum_{i=1}^n y_i \bm{K}_i \bm{a}_k+ \alpha_k (\bm{a}_k-\bm{a}'_k)^\top\bm{K}(\bm{a}_k-\bm{a}'_k)\Big\},
\]
where $\alpha_k$ is a proximal regularization parameter, and it is easy to get the closed-form solution as follows:
\begin{equation*}
  \bm{a}_k = \bm{a}'_k+{[y_1, y_2, \cdots, y_n]^\top}/{2\alpha_k n}\ .
\end{equation*}
We get the final $\bm{a}_k=\bm{a}_k/\|\bm{a}_k\|_{\bm{K}}$ in the $k$-th iteration.

\begin{table*}[ht]
  \caption{Comparisons of the test accuracies (mean$\pm$std.) on 20 datasets. We use Gaussian kernel for all algorithms. $\bullet/\circ$ indicates that our {M\scriptsize SV\small M\scriptsize Av} approach is significantly better/worse than the corresponding algorithms (pairwise t-tests at 95\% significance level).}\label{table:KPDM}
  \begin{center}
  \begin{small}
  \setlength{\tabcolsep}{1.6mm}{
  \begin{tabular}{|cccccccc|}
    \hline
    Dataset&{M\scriptsize SV\small M\scriptsize Av}&SVM&SVR&LSSVM&ODM&MAMC&FMM\\ \hline
advertise &.9837$\pm$.0014&.9820$\pm$.0023$\bullet$&.9835$\pm$.0019&.9848$\pm$.0016$\circ$&.9838$\pm$.0016&.9547$\pm$.0026$\bullet$&.9810$\pm$.0028$\bullet$\\
australian     &.8580$\pm$.0078&.8225$\pm$.0087$\bullet$&.8210$\pm$.0080$\bullet$&.8357$\pm$.0111$\bullet$&.8329$\pm$.0088$\bullet$&.8302$\pm$.0122$\bullet$&.8237$\pm$.0060$\bullet$\\
bibtex         &.7554$\pm$.0036&.7506$\pm$.0050$\bullet$&.7508$\pm$.0053$\bullet$&.7505$\pm$.0056$\bullet$&.7570$\pm$.0045&.6714$\pm$.0397$\bullet$&.7509$\pm$.0050$\bullet$\\
biodeg    &.8834$\pm$.0081&.8687$\pm$.0115$\bullet$&.8580$\pm$.0089$\bullet$&.8712$\pm$.0093$\bullet$&.8845$\pm$.0095&.8559$\pm$.0098$\bullet$&.8915$\pm$.0096$\circ$\\
breastw        &.9783$\pm$.0013&.9710$\pm$.0013$\bullet$&.9703$\pm$.0050$\bullet$&.9713$\pm$.0018$\bullet$&.9710$\pm$.0013$\bullet$&.9774$\pm$.0022$\bullet$&.9710$\pm$.0013$\bullet$\\
diabetes       &.7576$\pm$.0074&.7574$\pm$.0094&.7502$\pm$.0096$\bullet$&.7411$\pm$.0112$\bullet$&.7504$\pm$.0082$\bullet$&.6779$\pm$.0197$\bullet$&.7385$\pm$.0105$\bullet$\\
emotions       &.7986$\pm$.0122&.7899$\pm$.0156$\bullet$&.7756$\pm$.0164$\bullet$&.8115$\pm$.0124$\circ$&.8101$\pm$.0147$\circ$&.7952$\pm$.0120&.8078$\pm$.0126$\circ$\\
german   &.7465$\pm$.0099&.7443$\pm$.0096&.7198$\pm$.0174$\bullet$&.7353$\pm$.0143$\bullet$&.7285$\pm$.0156$\bullet$&.7198$\pm$.0154$\bullet$&.7277$\pm$.0127$\bullet$\\
halloffame&.9677$\pm$.0025&.9625$\pm$.0020$\bullet$&.9578$\pm$.0053$\bullet$&.9523$\pm$.0060$\bullet$&.9617$\pm$.0025$\bullet$&.9510$\pm$.0028$\bullet$&.9590$\pm$.0034$\bullet$\\
hill-valley   &.6826$\pm$.0184&.6668$\pm$.0177$\bullet$&.6482$\pm$.0135$\bullet$&.7116$\pm$.0678$\circ$&.5950$\pm$.0360$\bullet$&.5402$\pm$.0320$\bullet$&.7886$\pm$.0299$\circ$\\
kc1            &.8739$\pm$.0042&.8701$\pm$.0057$\bullet$&.8689$\pm$.0045$\bullet$&.8703$\pm$.0044$\bullet$&.8716$\pm$.0055&.8764$\pm$.0027$\circ$&.8706$\pm$.0038$\bullet$\\
parkinsons     &.9573$\pm$.0223&.9282$\pm$.0203$\bullet$&.9393$\pm$.0193$\bullet$&.9393$\pm$.0224$\bullet$&.9385$\pm$.0157$\bullet$&.9214$\pm$.0174$\bullet$&.9402$\pm$.0167$\bullet$\\
pbcseq         &.7350$\pm$.0143&.7214$\pm$.0152$\bullet$&.7317$\pm$.0151&.7312$\pm$.0165&.7269$\pm$.0221$\bullet$&.7076$\pm$.0149$\bullet$&.7238$\pm$.0184$\bullet$\\
sleepdata&.7407$\pm$.0129&.7192$\pm$.0105$\bullet$&.7211$\pm$.0061$\bullet$&.7037$\pm$.0083$\bullet$&.7050$\pm$.0083$\bullet$&.7055$\pm$.0056$\bullet$&.7182$\pm$.0094$\bullet$\\
students&.8977$\pm$.0119&.8920$\pm$.0079$\bullet$&.8665$\pm$.0098$\bullet$&.8898$\pm$.0072$\bullet$&.8805$\pm$.0142$\bullet$&.6543$\pm$.0185$\bullet$&.8993$\pm$.0062\\
titanic        &.7825$\pm$.0048&.7823$\pm$.0052&.7823$\pm$.0052&.7823$\pm$.0052&.7767$\pm$.0075$\bullet$&.7823$\pm$.0052&.7823$\pm$.0052\\
tokyo1         &.9406$\pm$.0037&.9241$\pm$.0050$\bullet$&.9257$\pm$.0060$\bullet$&.9337$\pm$.0054$\bullet$&.9253$\pm$.0053$\bullet$&.9229$\pm$.0039$\bullet$&.9248$\pm$.0050$\bullet$\\
vehicle        &.9793$\pm$.0083&.9795$\pm$.0090&.9856$\pm$.0073$\circ$&.9899$\pm$.0070$\circ$&.9722$\pm$.0088$\bullet$&.9625$\pm$.0105$\bullet$&.9805$\pm$.0093\\
vertebra&.8280$\pm$.0240&.7898$\pm$.0098$\bullet$&.7957$\pm$.0183$\bullet$&.8108$\pm$.0176$\bullet$&.8000$\pm$.0122$\bullet$&.7769$\pm$.0126$\bullet$&.7962$\pm$.0153$\bullet$\\
wdbc           &.9819$\pm$.0081&.9810$\pm$.0056&.9772$\pm$.0066$\bullet$&.9842$\pm$.0035&.9795$\pm$.0052&.9526$\pm$.0089$\bullet$&.9526$\pm$.0089$\bullet$\\
\hline
\multicolumn{2}{|c}{Win/Tie/Loss}&15/5/0&16/3/1&13/3/4&14/5/1&17/2/1&14/3/3\\

\hline
\end{tabular}}

\end{small}
\end{center}
\end{table*}

\section{Empirical Study}\label{sec:exp}

In this section, we present extensive empirical studies to verify the effectiveness of our proposed {M\scriptsize SV\normalsize M\scriptsize Av} approach.  We consider 30 datasets, including 20 regular and 10 large-scale datasets. The number of instances varies from 208 to 88588 while the feature dimensionality ranges from 2 to 1836, covering a wide range of properties. The statistics for all datasets can be found in Appendix~\ref{app:experiment_setup}.

We compare our proposed {M\scriptsize SV\normalsize M\scriptsize Av} approach with state-of-the-art algorithms on large-margin and margin distribution optimization: 1)~SVM \cite{Boser:Guyon:Vapnik1992}, 2) SVR \cite{Drucker:Burges:Kaufman:Smola:Vapnik1997} with binary targets, 3) LSSVM \cite{Suykens:Gestel:Brabanter:Moor:Vandewalle2002}, 4) MAMC \cite{Pelckmans:Suykens:DeMoor2007}, 5) ODM \cite{Zhang:Zhou2020}, 6) FMM \cite{Ji:Srebro:Telgarsky2021}. The details of compared algorithms can be found in Appendix~\ref{app:Related}.

For each dataset, we scale all features into the interval $[0, 1]$, and augment each instance $\x$ with constant $1$ for the bias of linear model. The empirical average margin $\theta_{\w}$ may be smaller than zero in experiments, when the proximal regularization parameter $\beta_k$ is set too small.  In such case, we take the opposite model $-\w$ so as to keep the positiveness of empirical average margin.

For our \MSVMAV approach, parameters $\alpha_k$ and $\beta_k$ are set to be constant and selected by $5$-fold cross validation from $\{2^{-10}, 2^{-8}, \cdots, 2^{10}\}$, and the width of Gaussian kernel is chosen from $\{2^{-10}/d, 2^{-8}/d, \cdots, 2^{10}/d\}$. We select the maximum iteration number $T=100$ as a stopping criteria for \MSVMAV. For SVM, SVR, LSSVM and ODM, we set regularization parameter $C\in\{2^{-10}, 2^{-8}, \cdots, 2^{10}\}$ by $5$-fold cross validation again, and the others are set according to their respective references, also shown in Appendix~\ref{app:Related}.

We first compare the margin distributions of our proposed \MSVMAV approach with other algorithms. Figure~\ref{fig:margin} illustrates the cumulative margin distributions of different algorithms on four datasets, and similar trends can be observed on other datasets. As can be seen, the curves of our \MSVMAV approach generally lie on the rightmost side, which shows the margin distributions of \MSVMAV are generally better than that of SVM, ODM and FMM.

We further analyze the generalization performance of our proposed \MSVMAV approach with other compared algorithms. All algorithms are evaluated by $30$ times of random partitions of datasets with $80\%$ and $20\%$ of data for training and testing, respectively. The test accuracies are obtained by averaging over $30$ times. Tables~\ref{table:PDM} and \ref{table:KPDM} show the empirical results of our {M\scriptsize SV\normalsize M\scriptsize Av} and other algorithms with linear and Gaussian kernel functions, respectively.

From Tables~\ref{table:PDM} and \ref{table:KPDM}, our proposed \MSVMAV approach takes significantly better performance than other algorithms for linear and kernel functions, since win/tie/loss counts show that our approach wins for most datasets, and rarely losses. One intuitive explanation is that our \MSVMAV approach achieves better margin distribution by maximizing the empirical average margin and minimizing empirical margin semi-variance, as shown in Figure~\ref{fig:margin}. SVM, SVR and FMM maximize the minimum margin, which ignores the margin distribution. LSSVM and MAMC essentially maximize average margin only, which fails to learn from other margin statistics. ODM takes the average margin and margin variance into consideration, but the process of margin variance minimization could constrain some large margins.

This section omits partial empirical results due to the page limit, including the empirical curves of margin distributions, as well as running time comparisons for our \MSVMAV and other compared algorithms. Relevant results can be found in Appendix~\ref{app:further_result}.

\section{Conclusion}\label{sec:con}
Large margin has been one of the most important principles on the design of algorithms in machine learning, and recent empirical studies show new insights on the optimization of  margin distribution yet without theoretical supports. This work takes one step on this direction by providing a new generalization error bound, which is heavily relevant to margin distribution by incorporating factors such as average margin and semi-variance. Based on the theoretical results, we develop the {M\scriptsize SV\normalsize M\scriptsize Av} approach for margin distribution optimization, and extensive experiments verify its superiority. An interesting future work is to exploit more effective  statistics to characterize the whole margin distribution.

\section*{Acknowledgements}
The authors want to thank the reviewers for helpful comments and suggestions. This work is partially supported by the NSFC (61921006, 61876078, 62006088). Wei Gao is the corresponding author.

\bibliographystyle{named}
\bibliography{mardis}

\onecolumn
\appendix
\begin{center}
{\Large\textbf{Supplementary Material (Appendix)}}
\end{center}

\section{Proof of Theorem~\ref{thm:Main}}\label{app:proof}
We first begin with some useful lemmas.
\begin{lemma}\label{lem:Appendix_McDiarmid}
  (McDiarmid's inequality \cite{Bartlett:Mendelson2002}) Let $X_1, X_2, \cdots, X_n$ be independent random variables taking values in a set A, and assume that $f\colon A^n\to \mathbb{R}$ satisfies
  \[
    \sup_{x_1, x_2, \cdots, x_n, x'_i\in A} \Big|f(x_1, x_2, \cdots, x_n)-f(x_1, x_2, \cdots, x_{i-1}, x'_i, x_{i+1}, \cdots, x_n)\Big|\leq c_i
  \]
  for every $i\in[n]$. Then, for every $t>0$,
  \[
    \Pr\Big[f(X_1, X_2, \cdots, X_n)-E\big[f(X_1, X_2, \cdots, X_n)\big]\geq t\Big]\leq e^{-2t^2/\sum_{i=1}^n c_i^2}\ ,
  \]
  where $e=2.718\cdots$ denotes the Euler's number.
\end{lemma}

\

By applying McDiarmid's inequality, we propose Lemma~\ref{lem:Appendix_Rademacher} to prove the uniform convergence of our squared margin loss.
\begin{lemma}\label{lem:Appendix_Rademacher}
Let $\H=\{h\colon \X\to[-1,+1]\}$ be a real-valued function space, and $\ell_\theta$ denote the squared margin loss for some constant $\theta$. Then, for every $h\in\H$ and $S_n\sim \D$, the following holds :
\[
  \Pr_{S_n}\left[\sup_{h \in\H}\Big\{\mathcal{E}(h)-E_{(\x, y)\sim S_n}\big[\ell_\theta((\x, y);h)\big]\Big\}\geq \left(\frac{4}{\theta}+\frac{4}{\theta^2}\right)\widehat{\mathfrak{R}}_{S_n}(\H) + 3\epsilon \right] \leq 2\exp\left(\frac{-2\epsilon^2\theta^4n}{(1+\theta)^4}\right)\ .
\]
\end{lemma}
\begin{proof}
  For simplicity, we define $\phi_\theta(t)=(1-t/\theta)_+^2(-1\leq t\leq 1)$ and we write $\widetilde{\H}=\{(\x, y)\to yh(\x)\colon h\in \H\}$. Then, we have $\ell_\theta((\x, y); h)=\phi_\theta(yh(\x))$, and the function space of the composition of $\ell_\theta$ and $\H$ can be given by
  \[
    \F_\theta=\Big\{\phi_\theta \circ \tilde{h}\colon \tilde{h} \in \widetilde{\H}\Big\}=\Big\{\ell_\theta((\x, y); h)\colon h\in \H\Big\}\ .
  \]
  For simplicity, we define
  \[
    \Phi_\theta(S_n) = \sup_{f \in\F_\theta}\Big\{E_{(x, y)\sim \D}\big[f(\x, y)\big]-E_{(x, y)\sim S_n}\big[f(\x, y)\big]\Big\}\ ,
  \]
  and thus, we have
  \[
    \Phi_\theta(S_n) - \Phi_\theta(S'_n)\leq \sup_{f \in\F_\theta}\Big\{E_{(x, y)\sim S'_n}\big[f(\x, y)\big] - E_{(x, y)\sim S_n}\big[f(\x, y)\big]\Big\}= \frac{1}{n}\sup_{f \in\F_\theta}\Big\{f(\x_j, y_j) - f(\x'_j, y'_j)\Big\}\ ,
  \]
  where $S'_n$ denotes the training sample different from $S_n$ by only one instance with index $j$: $(\x_j, y_j)$ in $S_n$ and $(\x'_j, y'_j)$ in $S'_n$.

  Since $h\in \H$ maps any instance $(\x, y)$ into the interval $[-1, 1]$, and $y_i\in \{-1, 1\}$, we could obtain $\tilde{h}(\x_i, y_i)\in [-1, 1]$ for any $x_1, x_2, \cdots, x_n, x'_j$, and any $\tilde{h} \in \widetilde{\H}$, which yields
  \[
    \Phi_\theta(S_n) - \Phi_\theta(S'_n)\leq\frac{1}{n}\sup_{f \in\F_\theta}\Big\{f(\x_j, y_j) - f(\x'_j, y'_j)\Big\}=\frac{1}{n}\sup_{\tilde{h} \in \widetilde{\H}}\Big\{\phi_\theta(\tilde{h}(\x_j, y_j))-\phi_\theta(\tilde{h}(\x'_j, y'_j))\Big\}\leq \frac{1}{n}\left(1+\frac{1}{\theta}\right)^2\ .
  \]
  Similarly, we can obtain $\Phi_\theta(S'_n) - \Phi_\theta(S_n)\leq \frac{1}{n}\left(1+\frac{1}{\theta}\right)^2$, thus $ |\Phi_\theta(S_n) - \Phi_\theta(S'_n)|\leq \frac{1}{n}\left(1+\frac{1}{\theta}\right)^2$.

  By applying the McDiarmid inequality, we could get
  \begin{equation}\label{eq:Appendix_McDiarmid_Bound}
    \Pr_{S_n}\Big[\Phi_\theta(S_n)-E_{S_n}\big[\Phi_\theta(S_n)\big]\geq \epsilon \Big]\leq \exp\left(\frac{-2\epsilon^2\theta^4n}{(1+\theta)^4}\right)\ .
  \end{equation}

  We next bound the term $E_{S_n}\big[\Phi_\theta(S_n)\big]$ as follows:
  \begin{eqnarray*}
    E_{S_n}\big[\Phi_\theta(S_n)\big]&=&E_{S_n}\left[\sup_{f \in\F_\theta}\Big\{E_{(x, y)\sim \D}\big[f(\x, y)\big]-E_{(x, y)\sim S_n}\big[f(\x, y)\big]\Big\}\right]\\
    &=&E_{S_n}\left[\sup_{f \in\F_\theta}\Big\{E_{S'_n\sim \D}\big[E_{(\x, y)\sim S'_n}[f(\x, y)\big]-E_{(x, y)\sim S_n}\big[f(\x, y)]\big]\Big\}\right]\\
    &\leq&E_{S_n, S'_n}\left[\sup_{f \in\F_\theta}\Big\{E_{(\x, y)\sim S'_n}[f(\x, y)\big]-E_{(x, y)\sim S_n}\big[f(\x, y)]\Big\}\right]\\
    &=&E_{S_n, S'_n}\left[\sup_{f \in\F_\theta}\frac{1}{n}\sum_{i=1}^n\left(f(\x'_i, y'_i)-f(\x_i, y_i)\right)\right]\\
    &=&E_{S_n, S'_n, \sigma_1, \sigma_2, \cdots, \sigma_n}\left[\sup_{f \in\F_\theta}\frac{1}{n}\sum_{i=1}^n\sigma_i\left(f(\x'_i, y'_i)-f(\x_i, y_i)\right)\right]\\
    &\leq&E_{S'_n, \sigma_1, \sigma_2, \cdots, \sigma_n}\left[\sup_{f \in\F_\theta}\frac{1}{n}\sum_{i=1}^n\sigma_if(\x'_i, y'_i)\right]+E_{S_n, \sigma_1, \sigma_2, \cdots, \sigma_n}\left[\sup_{f \in\F_\theta}\frac{1}{n}\sum_{i=1}^n-\sigma_if(\x_i, y_i)\right]\\
    &=&2E_{S_n, \sigma_1, \sigma_2, \cdots, \sigma_n}\left[\sup_{f \in\F_\theta}\frac{1}{n}\sum_{i=1}^n\sigma_if(\x_i, y_i)\right]=2\mathfrak{R}_n(\F_\theta)\ ,
  \end{eqnarray*}
  where $\mathfrak{R}_n(\F_\theta)$ is the Rademacher complexity of function class $\F_\theta$ defined as $\mathfrak{R}_n(\F_\theta)=E_{S_n}\Big[\widehat{\mathfrak{R}}_{S_n}(\F_\theta)\Big]$. Thus, we could obtain the following inequality from Eqn.~\eqref{eq:Appendix_McDiarmid_Bound}:
  \begin{equation}\label{eq:Bound_Phi_Rademacher}
    \Pr_{S_n}\Big[\Phi_\theta(S_n)\geq 2\mathfrak{R}_n(\F_\theta) + \epsilon \Big]\leq \exp\left(\frac{-2\epsilon^2\theta^4n}{(1+\theta)^4}\right)\ .
  \end{equation}
  Recall the definition
  \[
    \widehat{\mathfrak{R}}_{S_n}(\mathcal{\F_\theta})=E_{\sigma_1,\sigma_2,\ldots,\sigma_n} \left[\sup_{f\in \F_\theta}\frac{1}{n}\sum_{i=1}^n \sigma_i f(\x_i)\right]\ ,
  \]
  and we have the following inequality:
  \begin{eqnarray*}
    \left|\widehat{\mathfrak{R}}_{S_n}(\mathcal{\F_\theta})-\widehat{\mathfrak{R}}_{S'_n}(\mathcal{\F_\theta})\right|&=&\left|E_{\sigma_1,\sigma_2,\ldots,\sigma_n} \left[\sup_{f\in \F_\theta}\left[\frac{1}{n}\sum_{i=1}^n \sigma_i f(\x_i)\right]-\sup_{f\in \F_\theta}\left[\frac{1}{n}\sum_{i=1}^n \sigma_i f(\x_i)-\frac{1}{n}\sigma_j f(\x_j)+\frac{1}{n}\sigma_j f(\x'_j)\right]\right]\right|\\
    &\leq&\left|E_{\sigma_1,\sigma_2,\ldots,\sigma_n} \left[\sup_{f\in \F_\theta}\left[\frac{1}{n}\sum_{i=1}^n \sigma_i f(\x_i)-\frac{1}{n}\sum_{i=1}^n \sigma_i f(\x_i)+\frac{1}{n}\sigma_j f(\x_j)-\frac{1}{n}\sigma_j f(\x'_j)\right]\right]\right|\\
    &=&\left|E_{\sigma_1,\sigma_2,\ldots,\sigma_n} \left[\sup_{f\in \F_\theta}\left[\frac{1}{n}\sigma_j f(\x_j)-\frac{1}{n}\sigma_j f(\x'_j)\right]\right]\right|\\
    &\leq&\frac{1}{n}\left(1+\frac{1}{\theta}\right)^2\ .
  \end{eqnarray*}
  By applying the McDiarmid inequality again, we could obtain
  \begin{equation}\label{eq:Bound_Rademacher}
    \Pr_{S_n}\Big[\mathfrak{R}_{n}(\mathcal{\F_\theta})\geq \widehat{\mathfrak{R}}_{S_n}(\mathcal{\F_\theta}) + \epsilon \Big] \leq \exp\left(\frac{-2\epsilon^2\theta^4n}{(1+\theta)^4}\right)\ .
  \end{equation}
  We use the union bound to combine inequalities Eqn.~\eqref{eq:Bound_Phi_Rademacher} and Eqn.~\eqref{eq:Bound_Rademacher}, which yields the following inequality:
  \[
    \Pr_{S_n}\Big[\Phi_\theta(S_n)\geq 2\widehat{\mathfrak{R}}_{S_n}(\mathcal{\F_\theta}) + 3\epsilon \Big] \leq 2\exp\left(\frac{-2\epsilon^2\theta^4n}{(1+\theta)^4}\right)\ .
  \]
  Note that the $\phi_\theta(t)=(1-t/\theta)_+^2$ is $(\frac{2}{\theta}+\frac{2}{\theta^2})$-Lipschitz, i.e., $|\phi_\theta(t_1)-\phi_\theta(t_2)|\leq (\frac{2}{\theta}+\frac{2}{\theta^2})|t_1-t_2|$ for any $t_1, t_2\in[-1, 1]$, and we could further bound the empirical Rademacher complexity of $\F_\theta$ as follows:
  \begin{eqnarray*}
    \widehat{\mathfrak{R}}_{S_n}(\mathcal{\F_\theta})&\leq& \left(\frac{2}{\theta}+\frac{2}{\theta^2}\right)\widehat{\mathfrak{R}}_{S_n}(\widetilde{\H})\\
    &=&\left(\frac{2}{\theta}+\frac{2}{\theta^2}\right)E_{\sigma_1,\sigma_2,\ldots,\sigma_n} \left[\sup_{\tilde{h} \in \widetilde{\H}}\left[\frac{1}{n}\sum_{i=1}^n \sigma_i \tilde{h}(\x_i)\right]\right]\\
    &=&\left(\frac{2}{\theta}+\frac{2}{\theta^2}\right)E_{\sigma_1,\sigma_2,\ldots,\sigma_n} \left[\sup_{h \in \H}\left[\frac{1}{n}\sum_{i=1}^n \sigma_i y_ih(\x_i)\right]\right]\\
    &=&\left(\frac{2}{\theta}+\frac{2}{\theta^2}\right)E_{\sigma_1,\sigma_2,\ldots,\sigma_n} \left[\sup_{h \in \H}\left[\frac{1}{n}\sum_{i=1}^n \sigma_i h(\x_i)\right]\right]\\
    &=&\left(\frac{2}{\theta}+\frac{2}{\theta^2}\right)\widehat{\mathfrak{R}}_{S_n}(\H)
  \end{eqnarray*}
  which could yield the following inequality:
  \begin{equation}\label{eq:Bound_Phi}
    \Pr_{S_n}\left[\Phi_\theta(S_n)\geq \left(\frac{4}{\theta}+\frac{4}{\theta^2}\right)\widehat{\mathfrak{R}}_{S_n}(\H) + 3\epsilon \right] \leq 2\exp\left(\frac{-2\epsilon^2\theta^4n}{(1+\theta)^4}\right)\ .
  \end{equation}
  From the definition of $\Phi_\theta(S_n)$, we could obtain
  \begin{eqnarray*}
    \Phi_\theta(S_n)&=&\sup_{f \in\F_\theta}\Big\{E_{(\x, y)\sim \D}\big[f(\x, y)\big]-E_{(\x, y)\sim S_n}\big[f(\x, y)\big]\Big\}\\
    &=&\sup_{h \in\H}\Big\{E_{(\x, y)\sim \D}\big[\ell_\theta((\x, y);h)\big]-E_{(\x, y)\sim S_n}\big[\ell_\theta((\x, y);h)\big]\Big\}\\
    &\geq&\sup_{h \in\H}\Big\{E_{(\x, y)\sim \D}\big[\I[yh(\x)\leq0\big]-E_{(\x, y)\sim S_n}\big[\ell_\theta((\x, y);h)\big]\Big\}\\
    &=&\sup_{h \in\H}\Big\{\mathcal{E}(h)-E_{(\x, y)\sim S_n}\big[\ell_\theta((\x, y);h)\big]\Big\}\ .
  \end{eqnarray*}
  Hence, we could finally get the conclusion of Lemma~\ref{lem:Appendix_Rademacher} as follows:
  \[
    \Pr_{S_n}\left[\sup_{h \in\H}\Big\{\mathcal{E}(h)-E_{(\x, y)\sim S_n}\big[\ell_\theta((\x, y);h)\big]\Big\}\geq \left(\frac{4}{\theta}+\frac{4}{\theta^2}\right)\widehat{\mathfrak{R}}_{S_n}(\H) + 3\epsilon \right] \leq 2\exp\left(\frac{-2\epsilon^2\theta^4n}{(1+\theta)^4}\right)\ .
  \]
\end{proof}

Now, we start our proof for Theorem~\ref{thm:Main}.
\begin{proof}
  We construct a finite sequence $a_0, a_1, \cdots, a_K$ such that $(1-\nu)/K\in(1/2n, 1/n)$ and
  \[
    a_i = \nu+\frac{(1-\nu)i}{K}, \quad \text{for}\quad i=0, 1, 2, \cdots, K\ .
  \]
  From Lemma~\ref{lem:Appendix_Rademacher}, we could obtain
  \[
    \Pr_{S_n}\left[\sup_{h \in\H_\nu}\Big\{\mathcal{E}(h)-E_{(\x, y)\sim S_n}\big[\ell_{a_i}((\x, y);h)\big]\Big\}\geq \left(\frac{4}{a_i}+\frac{4}{a_i^2}\right)\widehat{\mathfrak{R}}_{S_n}(\H_\nu) + 3\epsilon \right] \leq 2\exp\left(\frac{-2\epsilon^2a_i^4n}{(1+a_i)^4}\right) \leq 2\exp\left(\frac{-2\epsilon^2\nu^4n}{(1+\nu)^4}\right)
  \]
  for all $i=0, 1, 2, \cdots, K$. It is noteworthy that the function class investigated here is $\H_\nu$, which consists of relatively-good classifiers. By applying the union bound, the following inequality holds:
  \[
    \Pr_{S_n}\left[\mathop{\sup_{h \in\H_\nu}}_{0\leq i\leq K-1}\Big\{\mathcal{E}(h)-E_{(\x, y)\sim S_n}\big[\ell_{a_i}((\x, y);h)\big]\Big\}\geq \left(\frac{4}{a_i}+\frac{4}{a_i^2}\right)\widehat{\mathfrak{R}}_{S_n}(\H_\nu) + 3\epsilon \right] \leq 2K\exp\left(\frac{-2\epsilon^2\nu^4n}{(1+\nu)^4}\right)\ .
  \]
  For any $\theta_h\in[\nu, 1)$, there exists an $a_{i^*}$ such that
  \[
    a_{i^*-1}\leq \theta_h < a_{i^*}\ .
  \]
  By setting $\delta = 2K\exp\left(\frac{-2\epsilon^2\nu^4n}{(1+\nu)^4}\right)$, we have $\epsilon=\sqrt{(1+\nu)^4\ln(2K/\delta)/(2\nu^4 n)}$, and the following holds for every $h\in\H_\nu$ with probability at least $1-\delta$:
  \begin{equation}\label{eq:proof_1}
    \mathcal{E}(h)\leq E_{(\x, y)\sim S_n}\Big[\ell_{a_{i^*}}\big((\x, y);h\big)\Big] + \left(\frac{4}{a_{i^*}}+\frac{4}{a_{i^*}^2}\right)\widehat{\mathfrak{R}}_{S_n}(\H_\nu) + \frac{3(1+\nu)^2}{\nu^2}\sqrt{\frac{\ln\frac{2K}{\delta}}{2n}}\ .
  \end{equation}
  To further upper bound the term $E_{(\x, y)\sim S_n}[\ell_{a_{i^*}}((\x, y);h)]$, we need to consider each sample $(\x_i, y_i)\in S_n$. If $y_i h(\x_i)\geq a_{i^*}$, we have
  \[
    \ell_{a_{i^*}}\big((\x_i, y_i);h\big)=0= \left(1-\frac{y_i h(\x_i)}{\theta_h}\right)_+^2=\ell_{\theta_h}\big((\x_i, y_i);h\big)\ .
  \]
  If $\theta_h\leq y_i h(\x_i)< a_{i^*}$, $\ell_{\theta_h}((\x_i, y);h)=0$, and we have
  \[
    \ell_{a_{i^*}}\big((\x_i, y_i);h\big)=\left(1-\frac{y_i h(\x_i)}{a_{i^*}}\right)^2\leq \left(\frac{a_{i^*}-\theta_h}{a_{i^*}}\right)^2\leq\ell_{\theta_h}\big((\x_i, y_i);h\big)+\frac{(1-\nu)^2}{K^2\nu^2}\ .
  \]
  If $y_i h(\x_i)<\theta_h$, we could obtain
  \begin{eqnarray*}
    \ell_{a_{i^*}}\big((\x_i, y_i);h\big)&=&\left(1-\frac{y_i h(\x_i)}{a_{i^*}}\right)^2\\
    &=&\left(1-\frac{y_ih(\x_i)}{\theta_h} + \frac{y_ih(\x_i)}{\theta_h}-\frac{y_i h(\x_i)}{a_{i^*}}\right)^2\\
    &=&\left(1-\frac{y_ih(\x_i)}{\theta_h}\right)^2 +2 \left(1-\frac{y_ih(\x_i)}{\theta_h}\right)\left(\frac{y_ih(\x_i)}{\theta_h}-\frac{y_i h(\x_i)}{a_{i^*}}\right)+\left(\frac{y_ih(\x_i)}{\theta_h}-\frac{y_i h(\x_i)}{a_{i^*}}\right)^2\\
    &\leq&\ell_{\theta_h}\big((\x_i, y_i);h\big) + 2\left(1+\frac{1}{\nu}\right)\frac{1-\nu}{K\nu^2}+\frac{(1-\nu)^2}{K^2\nu^2}\ .
  \end{eqnarray*}
  In summary, we have
  \[
    E_{(\x, y)\sim S_n}\big[\ell_{a_{i^*}}((\x, y);h)\big]=\frac{1}{n}\sum_{i=1}^n \ell_{a_{i^*}}((\x_i, y_i);h)\leq E_{(\x, y)\sim S_n}\big[\ell_{\theta_h}((\x, y);h)\big]+ 2\left(1+\frac{1}{\nu}\right)\frac{1-\nu}{K\nu^2}+\frac{(1-\nu)^2}{K^2\nu^2}\ .
  \]
  By combining with eqn.~\eqref{eq:proof_1}, we could obtain that, the following holds for every $h\in\H_\nu$ with probability at least $1-\delta$:
  \begin{eqnarray*}
    \mathcal{E}(h)&\leq& E_{(\x, y)\sim S_n}\Big[\ell_{a_{i^*}}\big((\x, y);h\big)\Big] + \left(\frac{4}{a_{i^*}}+\frac{4}{a_{i^*}^2}\right)\widehat{\mathfrak{R}}_{S_n}(\H_\nu) + \frac{3(1+\nu)^2}{\nu^2}\sqrt{\frac{\ln\frac{2K}{\delta}}{2n}}\\
    &\leq& E_{(\x, y)\sim S_n}\big[\ell_{\theta_h}((\x, y);h)\big]+ 2\left(1+\frac{1}{\nu}\right)\frac{1-\nu}{K\nu^2}+\frac{(1-\nu)^2}{K^2\nu^2}+\left(\frac{4}{\theta_h}+\frac{4}{\theta_h^2}\right)\widehat{\mathfrak{R}}_{S_n}(\H_\nu) + \frac{3(1+\nu)^2}{\nu^2}\sqrt{\frac{\ln\frac{2K}{\delta}}{2n}}\\
    &\leq&E_{(\x, y)\sim S_n}\big[\ell_{\theta_h}((\x, y);h)\big]+ \frac{2+2\nu}{n\nu^3}+\frac{1}{n^2\nu^2}+\left(\frac{4}{\theta_h}+\frac{4}{\theta_h^2}\right)\widehat{\mathfrak{R}}_{S_n}(\H_\nu) + \frac{3(1+\nu)^2}{\nu^2}\sqrt{\frac{1}{2n}\ln\frac{4n}{\delta}}\\
    &=&\frac{\text{SV}(h)}{\theta_h^2}+O\left(\frac{\widehat{\mathfrak{R}}_n(\H_\nu)}{\theta_h^2}+\sqrt{\frac{1}{2n}\ln\frac{4n}{\delta}} \right)\ ,
  \end{eqnarray*}
  which completes the proof.
\end{proof}

\section{Initialization of {M\small SV\Large M\small Av}}\label{app:formulation}
This section introduces the detailed proof for the closed-form solution in Eqn.~\eqref{eq:initial_primal} as follows:
\begin{proof}
The primal optimization problem for empirical average margin maximization is as follows:
\[
  \bm{w}_0=\mathop{\arg\max}_{\|\w\|^2_2=1} \sum_{i=1}^n \frac{y_i \langle\w,\x_i\rangle}{n}=\mathop{\arg\min}_{\|\w\|^2_2=1} -\sum_{i=1}^n \frac{y_i \langle\w,\x_i\rangle}{n}\ .
\]
It is necessary to introduce a Lagrange multiplier $\lambda\geq 0$ for the constraint $\left\|w\right\|_2^2=1$, and we could get the Lagrangian function
\[
  L(\w, \lambda) = -\frac{1}{n}\sum_{i=1}^n y_i \langle\w,\x_i\rangle + \lambda\left(\left\|\w\right\|_2^2 - 1\right)\ ,
\]
and its partial derivative is given by
\[
  \frac{\partial L(\w, \lambda)}{\partial \w} = -\frac{1}{n}\sum_{i=1}^n y_i \x_i + 2\lambda \w\ .
\]
By setting the partial derivative of $\w$ to zero, we could obtain
\[
  \w = \frac{1}{2\lambda n}\sum_{i=1}^n y_i \x_i\ ,
\]
which yields the dual problem as follows:
\[
  \max_{\lambda} -\frac{1}{n}\sum_{j=1}^n y_j \left\langle\frac{1}{2\lambda n}\sum_{i=1}^n y_i \x_i, \x_j\right\rangle + \lambda\left(\left\|\frac{1}{2\lambda n}\sum_{i=1}^n y_i \x_i\right\|_2^2 - 1\right)=-\frac{1}{4\lambda n^2}\left\|\sum_{i=1}^n y_i \x_i\right\|_2^2 -\lambda\ .
\]
It is obvious that the optimal value of $\lambda$ equals to
\[
  \lambda^*=\frac{1}{2n}\left\|\sum_{i=1}^n y_i \x_i\right\|_2\ .
\]
Hence, the empirical average margin maximizer, i.e., the initial value, is as follows:
\[
  \w_0 = \sum_{i=1}^n\frac{y_i\x_i}{\left\|\sum_{i=1}^n y_i \x_i\right\|_2}\ ,
\]
which completes the proof.
\end{proof}

\section{Detailed Information of Datasets}\label{app:experiment_setup}

\begin{table}
  \caption{Characteristic of 30 datasets.}
  \label{table:datasets}
  \begin{center}
  \begin{tabular}{|ccccccc|}
  \hline
  Scale&Dataset&\#Instance&\#Feature&Dataset&\#Instance&\#Feature\\
  \hline
  Regular&advertise&3279&1558&kc1&2109&21\\
  &australian&690&14&parkinsons&208&60\\
  &bibtex&7395&1836&pbcseq&1113&21\\
  &biodeg&1055&41&sleepdata&1024&2\\
  &breastw&683&9&students&1000&18\\
  &diabetes&768&8&titanic&2201&3\\
  &emotions&593&72&tokyo1&959&44\\
  &german&1000&24&vehicle&846&18\\
  &halloffame&1320&22&vertebra&310&6\\
  &hill-valley&1212&100&wdbc&569&30\\
  \hline
  Large&a9a&32561&123&magic&19020&10\\
  &acoustic&78823&50&nomao&34465&118\\
  &bank&45211&51&phishing&11055&68\\
  &eurgbp&43825&10&pol&15000&44\\
  &jm1&10880&21&run-walk&88588&6\\
  \hline
  \end{tabular}
  \end{center}
  \vskip -0.1in
\end{table}

This section introduces further experimental settings. We select 30 open-source datasets, including 20 regular scale and 10 large scale datasets, to verify the effectiveness of our proposed \MSVMAV approach. All datasets can be found on the UCI datasets website or the OpenML website. Table~\ref{table:datasets} shows the detailed information of each dataset.

\begin{figure}[!t]
  \centering
  \includegraphics[width=0.4\textwidth]{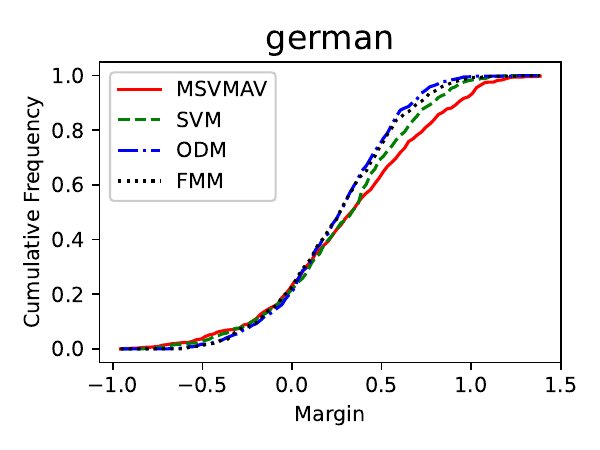}
  \includegraphics[width=0.4\textwidth]{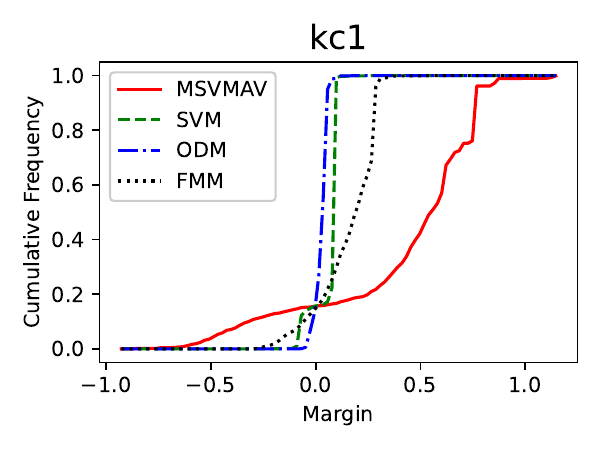}\\
  \includegraphics[width=0.4\textwidth]{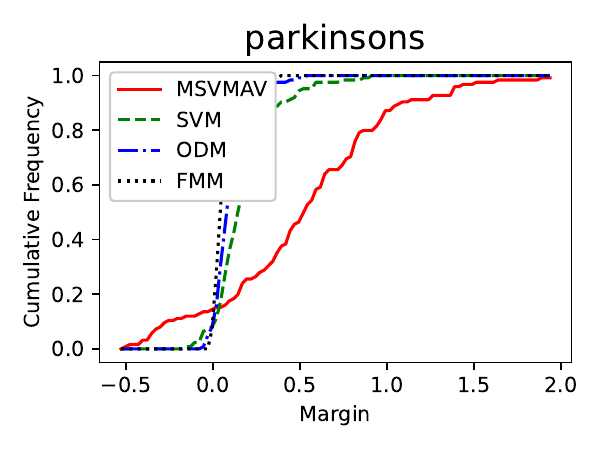}
  \includegraphics[width=0.4\textwidth]{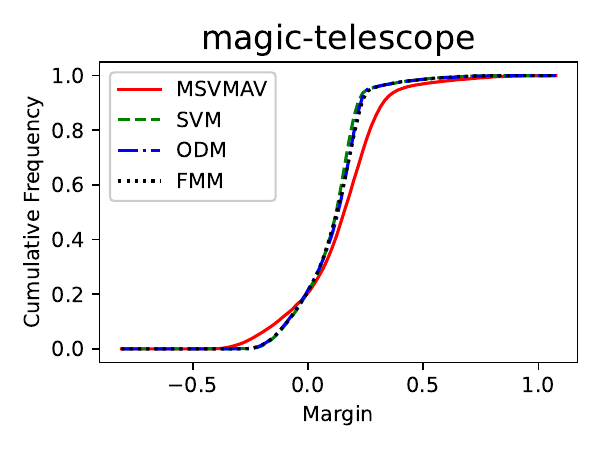}
  \centering
  \caption{Cumulative frequency versus margin of our {M\scriptsize SV\small M\scriptsize Av} and other algorithms such as SVM, ODM and FMM. The more right the curve, the better the margin distribution.}\label{fig:app_margin}
  \end{figure}

\section{Introductions to Other Algorithms}\label{app:Related}
This section introduces detailed information about six state-of-the-art algorithms that we compare our \MSVMAV approach with and the implementation of these algorithms. SVM (Support Vector Machine) is one of the most famous linear classification algorithms. Traditional hard-margin SVM maximizes the minimum margin. The classifier can be obtained via the following optimization problem:
\begin{eqnarray*}
  \min_{\w} & &\frac{1}{2}\|\w\|_2^2\\
  s.t.& &y_i \langle \w, \bm{\phi}(\x_i)\rangle \geq 1, \forall i\in[n].
\end{eqnarray*}
For non-separable data, the training data cannot be separated without error. Hence, the soft-margin SVM is proposed, which requires the solution of the following optimization problem:
\begin{eqnarray*}
  \min_{\bm{w}, b, \bm{\xi}}& &\frac{1}{2}\|\w\|_2^2 + C\sum_{i=1}^m \xi_i\\
  s.t.& &y_i(\langle\w, \bm{\phi}(\x_i)\rangle+b)\geq 1-\xi_i,\\
  & &\xi_i \geq 0,
\end{eqnarray*}
where $C>0$ is the penalty parameter for the error term, i.e., the hinge loss. On one hand, soft-margin SVM maximizes the minimum margin, while on the other hand, it tries to minimize the hinge loss. The optimization problems for SVR (Support Vector Regression) and LSSVM (Least-square Support Vector Machine) are similar to SVM's, while the only difference lies in the error term. SVR adopts the $\epsilon$-insensitive loss function while LSSVM uses the squared loss.

MAMC (Maximal Average Margin for Classifiers) and ODM (Optimal Margin Distribution Machine) considers to optimize the margin distribution directly. MAMC \cite{Pelckmans:Suykens:DeMoor2007} maximizes the average margin. The algorithm is very time-efficient because the optimization problem has closed-form solution. However, the experimental results indicate that the overall performance is the worst in most situations since MAMC considers the average margin only. ODM \cite{Zhang:Zhou2020} considers to optimize the empirical average margin and the empirical margin variance simultaneously. However, ODM introduces three hyper-parameters, which makes it difficulty to tune these parameters.

Recently, \citeauthor{Ji:Srebro:Telgarsky2021}~\shortcite{Ji:Srebro:Telgarsky2021} has proposed a new method to directly maximize the minimum margin, i.e., $\min_{i} y_i \langle \w, \x_i \rangle / \|\w\|_2^2$, via momentum-based gradient method. We call this method as FMM (Fast Margin Maximization) in this paper for short.

For linear SVM and SVR, we use the LIBLINEAR implementation, while for kernel SVM and SVR, we use the LIBSVM implementation instead\cite{Fan:Chang:Hsieh:Wang:Lin2008}. We use the LS-SVMlib\footnote{\url{https://www.esat.kuleuven.be/sista/lssvmlab/}} to implement LSSVM\cite{Suykens:Gestel:Brabanter:Moor:Vandewalle2002}. We use the source code from authors' website\footnote{\url{http://www.lamda.nju.edu.cn/code_ODM.ashx}} to implement ODM\cite{Zhang:Zhou2014}. We implement MAMC\cite{Pelckmans:Suykens:DeMoor2007} and FMM\cite{Ji:Srebro:Telgarsky2021} exactly according to their own respective work.

Apart from $\alpha_k$, $\beta_k$ for our {M\scriptsize SV\normalsize M\scriptsize Av} approach, $C$ for SVM, SVR, LSSVM, ODM and the width of Gaussian kernel, the parameter $\mu$ and $\theta$ for ODM are selected from $\{0.2, 0.4, 0.6, 0.8\}$ according to \cite{Zhang:Zhou2020}. For FMM, $\theta$, which can be viewed as the learning rate, is selected from the set $\{2^{-10}, 2^{-8}, \cdots, 2^0\}$ and $T$, which controls the maximum iteration number, is selected from $\{50, 100, 200, 400, 800, 1600\}$.

\section{Further Empirical Results}\label{app:further_result}
Due to page limit, we present more experimental results and running time comparison in this section.

\begin{figure}[!t]
  \centering
  \includegraphics[width=1\linewidth]{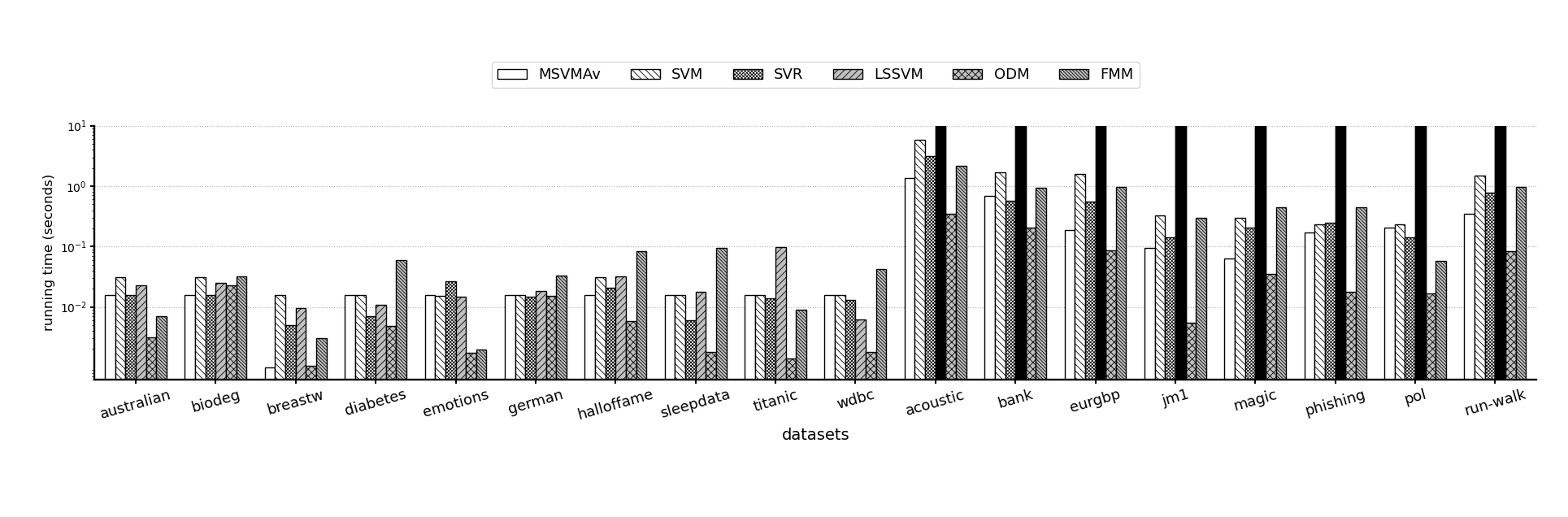}
  \caption{Full training time of our {M\scriptsize SV\small M\scriptsize Av}, SVM, SVR, LSSVM, ODM and FMM.}\label{fig:app_running_time}
\end{figure}

Figure~\ref{fig:app_margin} illustrates the curves of the cumulative margin distributions for our {M\scriptsize SV\normalsize M\scriptsize Av} approach and other algorithms on four more datasets. As can be seen, our {M\scriptsize SV\normalsize M\scriptsize Av} approach takes the rightmost curve, indicating better margin distribution in comparisons with other algorithms.

We compare the training time on a training set of one partition of our method with SVM, SVR, LSSVM, ODM, and FMM. For fair comparisons, we run all experiments on the same computer with a single CPU (Intel Core i5-9500 @ 3.00GHz), 16GB memory and Windows 10 operating system. The running time (in seconds) of these methods on each dataset is shown in Figure~\ref{fig:app_running_time}, which shows that our \MSVMAV approach has comparable time cost to other methods.

\end{document}